\documentclass[11pt, times, english]{article}

\usepackage{amsmath}
\usepackage{graphicx}
\usepackage{subcaption}
\usepackage{float}

\usepackage{url}
\usepackage{amsthm}
\usepackage{bm}
\usepackage{amsfonts}
\usepackage{amssymb}
\usepackage{a4wide}
\usepackage{tabularx}
\usepackage{stmaryrd}
\usepackage{authblk}

\newtheorem{theorem}{Theorem}[section]

\newtheorem{corollary}{Corollary}[section]
\newtheorem{definition}{Definition}[section]

\theoremstyle{remark}
\newtheorem{remark}{Remark}[section]

\newcommand{\mydeg}{\mathrm{deg}}
\newcommand{\reals}{{\rm I\!R}}

\newcommand{\conv}{\mathrm{conv}}
\newcommand{\Newt}{\mathrm{Newt}}
\newcommand{\ENewt}{\mathrm{ENewt}}

\begin{document}
\title{Tropical Polynomial Division and Neural Networks}
\author{Georgios Smyrnis}
\author{Petros Maragos}
\affil{School of ECE, National Technical University of Athens, Athens, Greece}
\affil[]{\texttt{geosmirnis@gmail.com}, \texttt{maragos@cs.ntua.gr}}
\date{}
\maketitle

\begin{abstract}
In this work, we examine the process of Tropical Polynomial Division, a geometric method which seeks to emulate the division of regular polynomials, when applied to those of the max-plus semiring. This is done via the approximation of the Newton Polytope of the dividend polynomial by that of the divisor. This process is afterwards generalized and applied in the context of neural networks with ReLU activations. In particular, we make use of the intuition it provides, in order to minimize a two-layer fully connected network, trained for a binary classification problem. This method is later evaluated on a variety of experiments, demonstrating its capability to approximate a network, with minimal loss in performance.
\end{abstract}

\section{Introduction}\label{sec:intro}

Minimax algebra \cite{Cuni79} and tropical geometry \cite{MaSt15} are fields of mathematics with applications in a variety of domains, such as the analysis of dynamic systems \cite{Butk10}, \cite{BCOQ01}, \cite{Mara17}, optimization \cite{CGQ04}, \cite{AGG12}, idempotent functional analysis \cite{LMS01} and morphological methods for computer vision \cite{Mara05b}, \cite{Serr82}. Recent works \cite{ZNL18}, \cite{ChMa18} have expanded the links of this branch of mathematics in the domain of neural networks with piecewise linear activations, demonstrating a profound connection between the two.

It is apparent that further study is needed, given these new insights, in the role that this particular type of algebra plays, in the context of neural networks, in order to better identify the inner workings of the latter. Such an accomplishment would potentially have applications in the problem of network minimization, given that, as demonstrated by the results of \cite{SongWeightsConnections}, \cite{Luo_2017_ICCV}, pruning a network can lead to considerable improvements in both network size and runtime, without significant loss of accuracy. 

In this work, we seek to expand the link between these fields, by examining the process of Tropical Polynomial Division. The problem of factoring tropical polynomials is already studied in the one dimensional case \cite{TropMathSpeyerSturmfels}, \cite{CuMe80}, but the process presented here is novel, in that it attempts to emulate the division of regular polynomials. This is done by approximating the Newton Polytope of the dividend, which is characteristic of a tropical polynomial as a function \cite{ChMa17-}. Insight gained by this method is also applied in minimizing a two-layer neural network, trained for a binary classification problem, with preliminary experiments demonstrating its validity.

The rest of this work is structured as follows. In Section \ref{sec:trop}, we shall perform a review of certain key elements of tropical algebra, used in the rest of this work. In Section \ref{sec:div} we shall analyze our algorithm for tropical polynomial division, while in Section \ref{sec:appl} we shall examine how the intuition it provides can be applied in the case of neural network minimization. Finally, in Section \ref{sec:exper} we shall present certain experiments performed in this problem, and discuss their results.

\section{Related Work}\label{sec:related}

\subsection{Tropical Geometry}

The field of tropical algebra, as defined in \cite{MaSt15}, \cite{TropMathSpeyerSturmfels},  corresponds with the study of the $(\reals \cup \lbrace \infty \rbrace, \min,+)$ semiring, which is also referred to as min-plus algebra. The term is also used to refer to the dual version of this algebra \cite{ChMa18}, that is, max-plus algebra, where the semiring studied is instead $(\reals \cup \lbrace -\infty \rbrace, \max,+)$, and in this work the latter meaning will be assumed. Both types of algebras have been analyzed in the context of dynamic systems \cite{Butk10}, \cite{Mara17}, while in \cite{MaSt15} relations with algebraic geometry are explored, where ideas from tropical geometry are used in the study of varieties of Laurent polynomials. Moreover, systems of linear equations in this type of algebra have also been studied in \cite{Butk10} and \cite{TsMa19}.

It is also possible to define tropical polynomials, that is, multivariate polynomials that have coefficients belonging to the above semiring \cite{TropMathSpeyerSturmfels}, which support the same operations. These polynomials are piecewise linear functions, with their roots being the set of points where multiple linear regions correspond to the same maximum value. The analysis of this set of roots gives rise to the notion of tropical varieties and ideals, which are further examined in \cite{maclagan_rincon_2018}.

Recently, there has been work in linking max-plus algebra with the field of machine learning and neural networks. In particular, \cite{ZNL18} and \cite{ChMa18} use this type of algebra for the analysis of neural networks with ReLU activations as piecewise linear functions, and arrive at similar conclusions as \cite{NIPS2014_5422}, regarding the number of linear regions found in the network, with \cite{ChMa18} also providing an algorithm to sample these linear regions. Moreover, \cite{ChMa17-} and \cite{Zha+19} have demonstrated the use of morphological perceptrons, which rely on max-plus operations instead of normal inner products with the weights of the neuron, demonstrating particular ease when pruning neurons from the network, while also linking them with maxout networks \cite{GWM+13}.

\subsection{Posynomials and Geometric Programming}

Geometric programming involves the minimization of posynomials (Laurent polynomials restricted to positive values of the variables, and positive coefficients), under similar posynomial constraints. This type of problem applies to log-log convex functions, whose logarithm is convex in the logarithm of the variables \cite{Boyd_GeomProg_07}. This type of problem is therefore easily solvable, in the context of convex optimization.

Moreover, recent work has further analyzed the link between posynomials and a particular type of neural networks, called Log-Sum-Exp networks \cite{Calafiore2018LogsumexpNN}, \cite{Calafiore2019}. The latter have exponential activations in the hidden layer and logarithmic in the output neuron. The difference of two such networks is also proven to be a universal approximator, and can be used to model non-convex data, as a difference of convex functions.

\subsection{Neural Network Pruning}

The act of pruning neural networks attempts to approximate a given network's result, with one having a smaller number of neurons, by systematically removing parts of the network in order to decrease its size. This method has been long known and studied \cite{BackpropPruning}, in the context of back-propagation trained, fully connected neural networks. More recently, further work has been made in applying pruning techniques in convolutional layers of modern architectures, either by selecting the filter with the least important output \cite{Luo_2017_ICCV}, or by stochastically removing connections between neurons, and subsequently neurons themselves, from the network \cite{SongWeightsConnections}, both with remarkable results. The success of these methods suggests that the problems solved by many neural networks require much less computational effort, in order to perform equally well to the given task.

\section{Basic Elements of Tropical Algebra}\label{sec:trop}

In this section, we shall perform a review of the basic elements of Tropical Algebra, as demonstrated in previous works on the field. This review shall become the theoretical foundation for the algorithm we shall discuss in the following.

As we have already described, we shall think of tropical algebra in its max-plus version. In this type of algebra, it is possible to define polynomials, similar to those in normal algebra. In particular, any function of the form:

\begin{equation}
	p(\bm{x}) = \max\limits_{i=1}^{k} \left(\bm{a_i}^T \bm{x} + b_i\right) = \max\limits_{i=1}^{k} \left(\sum_{j=1}^{d} a_{ij} x_j + b_i\right), ~ \forall \bm{x} \in \reals^d
\end{equation}
is a \textit{tropical (multivariate) polynomial}. This function is piecewise linear, due to being equal to one of its linear terms, in the entirety of its domain.

Similar to regular polynomials, it is possible to define roots for tropical polynomials. In particular:

\begin{definition}
	The roots of a tropical polynomial $p(\bm{x})$ of the above form are defined as the points $\bm{x} \in \reals^d$, where there are two linear terms of the polynomial, corresponding to indices $i_1, i_2, i_1\neq i_2$, such that:
	\begin{gather}
		\bm{a_{i_1}}^T \bm{x} + b_{i_1} = \bm{a_{i_2}}^T \bm{x} + b_{i_2} \nonumber \\
		\bm{a_{i_1}}^T \bm{x} + b_{i_1} \geq \bm{a_{i}}^T \bm{x} + b_{i}, \forall i = 1 , \dots , k
	\end{gather}
\end{definition}

These roots correspond exactly to the points where the function represented by the tropical polynomial is not differentiable (since there is a sharp change in the gradient of the function, due to moving from one linear term to the other). Similar to regular polynomials, the set of all roots of a given tropical polynomial defines a \textit{tropical polynomial curve}. As an example, the roots of a tropical polynomial of degree 1, in two dimensions:
\begin{equation}
	p(x,y) = \max (x+b_1, y+b_2, b_3)
\end{equation}
form a \textit{tropical line} in two dimensions. The latter consists of three half-rays, starting from the point $(b_3-b_2, b_3-b_1)$, and emanating at $45^\circ, 180^\circ$ and $270^\circ$, respectively \cite{ChMa18}. This can be seen in Figure \ref{fig:trop_line}.

\begin{figure}[h]
	\centering
	\includegraphics[scale=0.7]{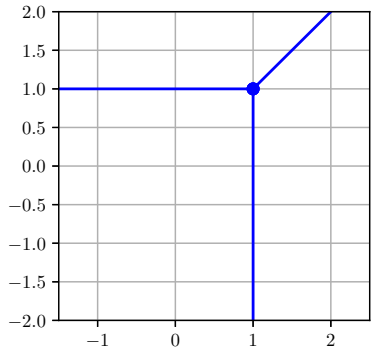}
	\caption{Tropical curve of $\max(x,y,1)$.}
	\label{fig:trop_line}
\end{figure}

An element intrinsic to the study of polytopes (and which will be central to what follows) is that of the \textit{Newton Polytope} of a tropical polynomial.

\begin{definition}[Newton Polytope \cite{ChMa17-}]
	Given a tropical polynomial $p(\bm{x})$, the associated \textit{Newton Polytope} is the set:
	\begin{equation}
	\Newt(p) = \conv \left\lbrace \bm{a}_i : i = 1 , 2 , \dots k \right\rbrace
	\end{equation}
	where $\conv S$ denotes the convex hull of the set $S$.
	
\end{definition}

\begin{definition}[Extended Newton Polytope \cite{ChMa17-}]
	Given the same tropical polynomial as above, the associated \textit{Extended Newton Polytope} is the set:
	\begin{equation}
	\ENewt(p) = \conv \left\lbrace \left( \bm{a}_i ,  b_i\right) : i = 1 , 2 , \dots k \right\rbrace
	\end{equation}
	
\end{definition}

Examples of both of these polytopes, for tropical polynomials in one and two dimensions, can be seen in Figures \ref{fig:newt_1d}, \ref{fig:newt_2d}. These polytopes have the following two properties, which are linked to the behavior of the tropical polynomials from which they are derived.

\begin{itemize}
	\item As noted in \cite{ZNL18}, \cite{ChMa18}, the tropical curve of a polynomial is directly linked to the above polytopes. In particular, the tropical curve is the dual version of the graph obtained when projecting the $\ENewt(p)$ in the first $d$ dimensions. Each linear region in $\ENewt(p)$ corresponds to a vertex in the tropical curve of the polynomial.
	
	\item It is well-known \cite{ZNL18}, \cite{ChMa18}, \cite{ChMa17-} that a tropical polynomial is defined as a function only by the terms corresponding to vertices on the upper faces of its Extended Newton Polytope. Indeed, terms which lie anywhere else on the polytope (in the interior, or at any other point of its faces) are a convex combination of points corresponding to these vertices, thus for any given $\bm{x} \in \reals^d$, terms corresponding to vertices will be larger.
\end{itemize}

\begin{figure}[h]
	\centering
	\captionsetup{justification=centering,margin=1cm}
	\begin{subfigure}[t]{0.5\textwidth}
		\centering
		\includegraphics[width=\textwidth]{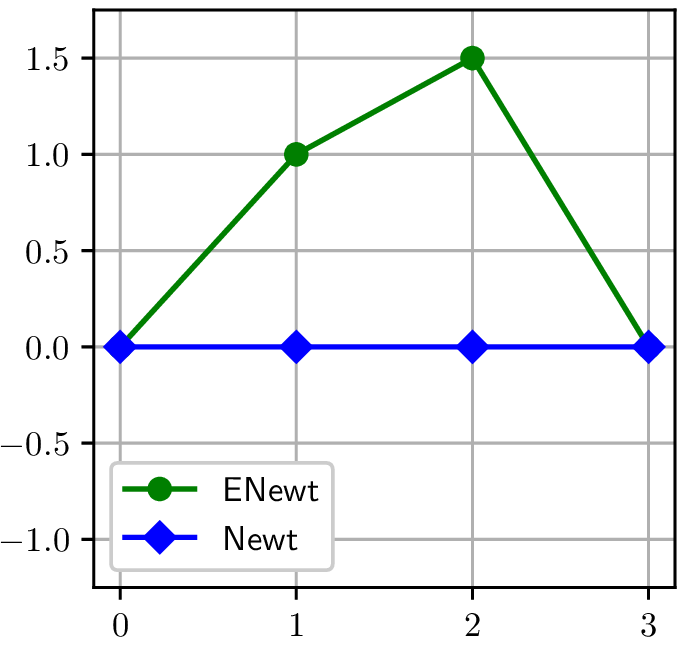}
		\caption{}
		\label{fig:newt_1d}
	\end{subfigure}%
	\begin{subfigure}[t]{0.5\textwidth}
		\centering
		\includegraphics[width=\textwidth]{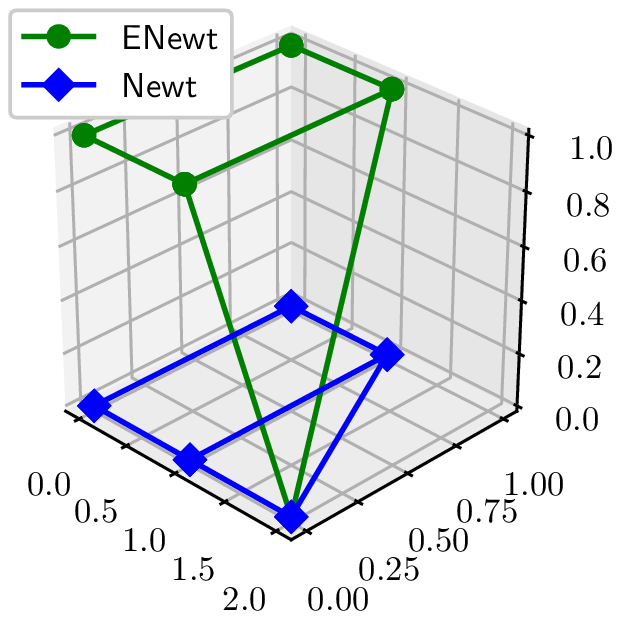}
		\caption{}
		\label{fig:newt_2d}
	\end{subfigure}
	\caption{(Extended) Newton Polytopes of (a) $\max(3x,2x+1.5,x+1,0)$ (b) $\max(2x,x+y+1,y+1,x+1,1)$.}
	\label{fig:newt}
\end{figure}

\section{Division of Tropical Polynomials}\label{sec:div}

\subsection{Tropical Polynomial Division Algorithm}

Let us now assume that we want to approximately divide a tropical polynomial $p(\bm{x}) = \max\limits_{i=1}^k \left( \bm{a}_i^T \bm{x} + b_i \right),~ \bm{x} \in \rm \reals^d$, by another tropical polynomial $d(\bm{x}) = \max\limits_{i=1}^k \left( \tilde{\bm{a}_i}^T \bm{x} + \tilde{b_i} \right)$. The algorithm proposed for this goal is an approximation, given that it outputs two tropical polynomials $q(\bm{x}) , ~ r(\bm{x})$, for which the following inequality holds:
\begin{equation}\label{eqn:ineq_1}
	p(\bm{x}) \geq \max \left( q(\bm{x}) + d(\bm{x}) , r(\bm{x}) \right) , ~ \forall \bm{x} \in \rm \reals^d
\end{equation}
The above two polynomials are maximal, in the sense that for any other two tropical polynomials $\tilde{q}(\bm{x}) , ~ \tilde{r}(\bm{x})$, with terms of the same degree, satisfying (\ref{eqn:ineq_1}), the following holds:
\begin{equation}\label{eqn:ineq_2}
	\max \left( q(\bm{x}) + d(\bm{x}) , r(\bm{x}) \right) \geq \max \left( \tilde{q}(\bm{x}) + d(\bm{x}) , \tilde{r}(\bm{x}) \right) , ~ \forall \bm{x} \in \reals^d 
\end{equation}

In what follows, we shall assume that every tropical polynomial contains a term for every point in the interior of its Newton Polytope (given the second property of Newton Polytopes described in the previous section, this does not affect the function corresponding to the tropical polynomial). Furthermore, we shall also assume that this term corresponds to a point that lies exactly on the upper hull of the Extended Newton Polytope (so as to be maximal, without altering the function).



\textbf{Tropical Polynomial Division Algorithm}.	Let $p(\bm{x}) , ~ d(\bm{x})$ be two tropical polynomials. The algorithm to divide $p(\bm{x})$ by $d(\bm{x})$ is as follows:
	\begin{itemize}
		\item Let $C$ be the set of slopes with which we can shift $\Newt(d)$, so that it is contained in the interior of $\Newt(p)$. More formally:
		\begin{equation}
			C = \left\lbrace \bm{c} \in \mathbb{Z}^d : \Newt\left( \bm{c}^T \bm{x} + d(\bm{x}) \right) \subseteq \Newt(p) \right\rbrace
		\end{equation}
		
		\item For every element $\bm{c} \in C$, define $q_c \in \reals$ as the largest value for which $\ENewt(p)$ is above $\ENewt\left(q + \bm{c}^T \bm{x} + d(\bm{x}) \right)$. This is equivalent to:
		\begin{equation}
			q_c = \max \left\lbrace q \in \reals : p(\bm{x}) \geq q + \bm{c}^T \bm{x} + d(\bm{x}) , ~ \forall \bm{x} \in \reals^d \right\rbrace
		\end{equation}
		
		\item Output the tropical polynomial:
		\begin{equation}
			q(\bm{x}) = \max\limits_{c \in C} \left( q_c + \bm{c}^T \bm{x} \right)
		\end{equation}
		along with the tropical polynomial:
		\begin{equation}
			r(\bm{x}) = \max\limits_{t(\bm{x}) \in T} \left( t(\bm{x})  \right)
		\end{equation}
		where $T$ is the set of terms $\bm{a}_j^T \bm{x} + b_j$ of $p(\bm{x})$ for which there is no value $\bm{c}$ such that $\bm{a}_j \in \Newt\left( \bm{c}^T \bm{x} + d(\bm{x}) \right)$.
	\end{itemize}	
As is the case of classic polynomial division, we shall refer to $q(\bm{x})$ as the \textit{quotient}, and $r(\bm{x})$ as the \textit{remainder} of the division of $p(\bm{x})$ by $d(\bm{x})$.

\theoremstyle{remark}
\begin{remark}
	The above algorithm is equivalent to the morphological opening of the upper hull of $\ENewt(p)$, which can be regarded as a function $n_p: \mathbb{Z}^d \to \reals$, by the upper hull of $\ENewt(d)$ (similarly $n_d: \mathbb{Z}^d \to \reals$). Indeed, the operation of morphological opening (\cite{RoHe91}, \cite{Mara98}) is the composition of a morphological erosion,
	\begin{equation}
		\epsilon(\bm{x}) = \inf\limits_{c \in \mathbb{Z}^d} \lbrace n_p (\bm{x}+c) - n_d(c)\rbrace
	\end{equation}
	where $C$ is the same set of valid shifts of the divisor, with a morphological dilation:
	\begin{equation}
		\delta\epsilon(\bm{x}) = \sup\limits_{c \in \mathbb{Z}^d} \lbrace \epsilon (\bm{x} - c) + n_d(c) \rbrace
	\end{equation}
	The erosion eliminates all points in $\Newt(p)$ which shifts in the divisor cannot match (setting their value to $-\infty$), while also lowering the vertices of the polygon. The dilation which follows restores these vertices, albeit in a way such that the polytope of the result can fit as high as possible, due to the infimum during the erosion leading to perfect reconstruction being impossible. The end result will be equivalent to the graph produced by fitting the function $n_d(\bm{x})$ as closely as possible to $n_p(\bm{x})$, which is the same operation as that of our algorithm. 

	Due to this equivalence, in contrast to the way division for regular polynomials works, it is not necessary to iterate over the elements of $C$ in any particular order, or even sequentially. Indeed, there is no cancellation of terms in the resulting polynomial, hence the division as described cannot have different results, based on the order in which elements of $C$ are considered (while this may result in examining terms which may prove to be unnecessary, they will have no effect on the function corresponding to the polynomial). Thus, the above algorithm can be parallelized in a straightforward fashion.
	
\end{remark}

As previously mentioned, the tropical polynomials produced by the above algorithm have the following properties:
\begin{theorem}
	For any tropical polynomials $p(\bm{x}),~d(\bm{x})$, the quotient $q(\bm{x})$ and the remainder $r(\bm{x})$ produced by the algorithm satisfy (\ref{eqn:ineq_1}).	
\end{theorem}

\begin{proof}
	For every term of the form $q_c + \bm{c}^T \bm{x} + d(\bm{x})$ produced by the algorithm, its Extended Newton Polytope is constructed so that its upper hull lies directly below that of $\ENewt(p)$. Furthermore:
	\begin{align}
		\max\limits_{c \in C} \left\lbrace q_c + \bm{c}^T \bm{x} + d(\bm{x}) \right\rbrace &= \max\limits_{c \in C} \left\lbrace q_c + \bm{c}^T \bm{x} \right\rbrace + d(\bm{x}) \nonumber \\
		&= q(\bm{x}) + d(\bm{x})
	\end{align}
	Given that the Extended Newton Polytope of the maximum of several terms is equivalent to the convex hull of the union of the separate polytopes, this means that the upper hull of $\ENewt\left( q(\bm{x}) + d(\bm{x}) \right)$ lies strictly below that of $\ENewt(p)$. This implies that:
	\begin{equation}\label{eq:proof_part1}
		\max \left( q(\bm{x}) + d(\bm{x}) , p(\bm{x}) \right) = p(\bm{x}), ~ \forall \bm{x} \in \rm \reals^d
	\end{equation}
	Furthermore, since the remainder $r(\bm{x})$ consists only of terms which already appear in $p(\bm{x})$, it is evident that:
	\begin{equation}\label{eq:proof_part2}
		\max \left( r(\bm{x}) , p(\bm{x}) \right) = p(\bm{x}), ~ \forall \bm{x} \in \rm \reals^d
	\end{equation}
	The required result immediately follows from \ref{eq:proof_part1} and \ref{eq:proof_part2}.
\end{proof}

\begin{theorem}
	The tropical polynomials $q(\bm{x}),~r(\bm{x})$ are maximal, in the sense that for any other two tropical polynomials $\tilde{q}(\bm{x}) , ~ \tilde{r}(\bm{x})$, for which (\ref{eqn:ineq_1}) holds, containing terms of the same degree as $q(\bm{x})$ and $r(\bm{x})$ respectively, the following holds:
	\begin{equation}
		\max \left( q(\bm{x}) + d(\bm{x}) , r(\bm{x}) \right) \geq \max \left( \tilde{q}(\bm{x}) + d(\bm{x}) , \tilde{r}(\bm{x}) \right) , ~ \forall \bm{x} \in \reals^d
	\end{equation}
\end{theorem}

\begin{proof}
	Since the two quotients and the two remainders contain terms of the same degrees, we can study each pair separately.
	\begin{itemize}
		\item Given that for each term of $q(\bm{x})$:
		\begin{equation}
			q_c = \max \left\lbrace q \in \reals : p(\bm{x}) \geq q + \bm{c}^T \bm{x} + d(\bm{x}) , ~ \forall \bm{x} \in \reals^d \right\rbrace
		\end{equation}
		it must be the case that $q_c \geq \tilde{q}_c , ~ \forall c \in C$, since both are lower than the divisor. Therefore:	
		\begin{equation}
			q(\bm{x}) = \max\limits_{c \in C} \left\lbrace q_c + \bm{c}^T \bm{x} \right\rbrace \geq \max\limits_{c \in C} \left\lbrace \tilde{q}_c + \bm{c}^T \bm{x}\right\rbrace = \tilde{q}(\bm{x}), \forall \bm{x} \in \reals^d
		\end{equation}
		
		\item Given that the remainder provided by the algorithm matches the divisor exactly, the same argument can be used to prove that $r(\bm{x}) \geq \tilde{r}(\bm{x})$.
	\end{itemize}
	The required property immediately follows.
\end{proof}

\subsection{Tropical Polynomial Division with Multiple Divisors}

As is the case with normal polynomial division, we can also extend the above algorithm in order to divide a tropical polynomial $p(\bm{x})$ by multiple tropical polynomials $d_i (\bm{x}) , i = 1 , \dots , n$. This means that we are now required to find $n$ quotients $q_i (\bm{x}) , i = 1 , \dots , n$, along with a remainder $r(\bm{x})$, so that:
\begin{equation}
	p(\bm{x}) \geq \max \left( \max\limits_{i=1}^n \left( q_i(\bm{x}) + d_i(\bm{x}) \right) , r(\bm{x})  \right) , ~ \forall \bm{x} \in \reals^d
\end{equation}

In contrast with the sequential approach of this problem when dealing with normal polynomials, the algorithm we propose for this goal is a direct extension of the one described in the previous section. In particular, the algorithm is as follows:

\textbf{Tropical Polynomial Division Algorithm - Multiple Divisors}	Let $p(\bm{x})$	be a tropical polynomial, along with a set of divisors:
\begin{equation}
	d_i (\bm{x}) = \max\limits_{j=1}^k \left( \tilde{\bm{a}}_{ij}^T \bm{x} + \tilde{b}_{ij} \right),~ \bm{x} \in \rm \reals^d , ~ i = 1 , \dots , n
\end{equation}
The algorithm to divide $p(\bm{x})$ by $d_i (\bm{x}) , i = 1 , \dots , n$ is as follows:
\begin{itemize}
	\item For each of the divisors $d_i (\bm{x})$, execute the single divisor algorithm, with corresponding outputs $q_i (\bm{x}) , r_i (\bm{x})$.
	
	\item The algorithm outputs as quotients the tropical polynomials $q_i (\bm{x})$, and as remainder the tropical polynomial $r(\bm{x})$ which consists of terms that appear in all remainders $r_i(\bm{x})$.
\end{itemize}

It is evident that, once again, the algorithm does not impose any particular order on how the divisors $d_i (\bm{x})$ are examined. This leads to the following property, of particular interest:
\begin{corollary}
	The remainder $r(\bm{x})$ produced by the algorithm is the same, regardless of the ordering of the divisors $d_i (\bm{x})$.
\end{corollary}
\noindent This is not the case for regular polynomial division, where the ordering of the divisors may affect the remainder (in particular, when the set of divisors is not a Groebner basis for its ideal \cite{Cox:2007:IVA:1204670}).

\subsection{Examples}

In order to illustrate the above algorithm, let us think of the following example in one dimension where:
\begin{equation}
	p(x) = \max \left(3x, 2x+1.5, x+1, 0\right) , ~ d(x) = \max \left(x+1, 0\right)
\end{equation}
The valid set of indices for the quotient, as described in the algorithm is $C = \left\lbrace 0, 1, 2 \right\rbrace$. Let us examine the case where $c = 1$. We can easily see that by raising the polytope of the polynomial:
\begin{equation}
	q_1 + x + \max \left(x+1, 0\right) = q_1 + \max(2x+1,x)
\end{equation}
the optimum can be found for $q_1 = 0.5$, for which the right vertex of the divisor coincides with the third vertex of the dividend. Similarly we can find that $q_0 = 0, q_2 = -1$. Thus we get:
\begin{equation}
	q(x) = \max \left( 2x-1, x+0.5, 0\right)
\end{equation}
We can verify that $p(x) = q(x) + d(x)$, as shown in Figure \ref{fig:div_1c}. However, in the case that:
\begin{equation}
p(x) = \max \left(3x, 2x+1.5, x+1, 0\right) , ~ d(x) = \max \left(x, 0\right)
\end{equation}
then for $c = 1$ we have:
\begin{equation}
q_1 + x + \max \left(x, 0\right) = q_1 + \max(2x,x)
\end{equation}
Here we get $q_1 = 1$, which, when combined with the results of the other degrees, leads to one of the vertices of the dividend being higher than those of the result, which causes the two being different, as seen in Figure \ref{fig:div_2c}.

\begin{figure}[h]
	\centering
	\begin{subfigure}[t]{0.3\textwidth}
		\centering
		\includegraphics[width=\textwidth]{./img/Figure_1.eps}
		\caption{}
		\label{fig:div_1a}
	\end{subfigure}%
	\begin{subfigure}[t]{0.3\textwidth}
		\centering
		\includegraphics[width=\textwidth]{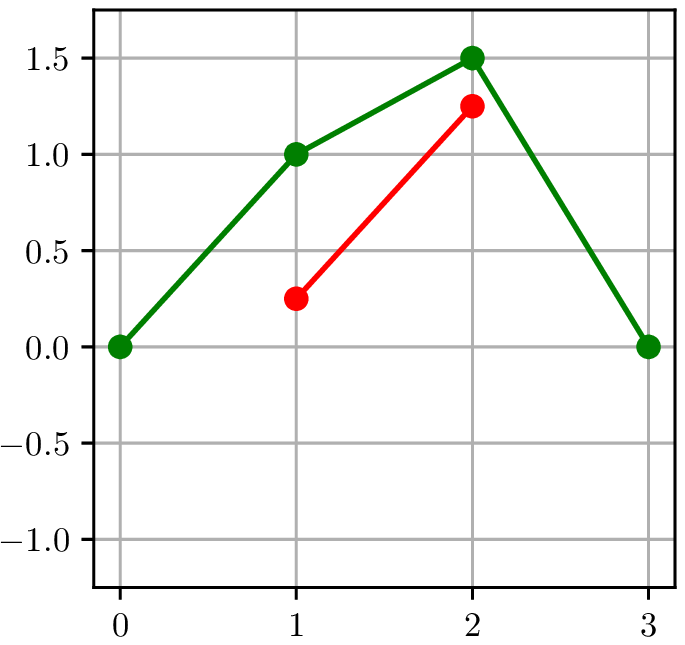}
		\caption{}
		\label{fig:div_1b}
	\end{subfigure}
	\begin{subfigure}[t]{0.3\textwidth}
		\centering
		\includegraphics[width=\textwidth]{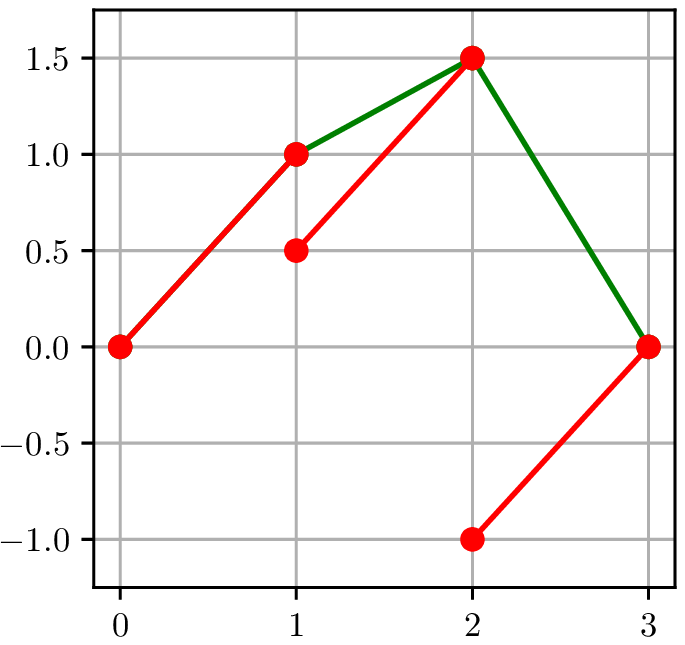}
		\caption{}
		\label{fig:div_1c}
	\end{subfigure}
	\caption{Division by $\max(x+1,0)$.}
	\label{fig:div_1}
\end{figure}

\begin{figure}[h]
	\centering
	\begin{subfigure}[t]{0.3\textwidth}
		\centering
		\includegraphics[width=\textwidth]{./img/Figure_1.eps}
		\caption{}
		\label{fig:div_2a}
	\end{subfigure}%
	\begin{subfigure}[t]{0.3\textwidth}
		\centering
		\includegraphics[width=\textwidth]{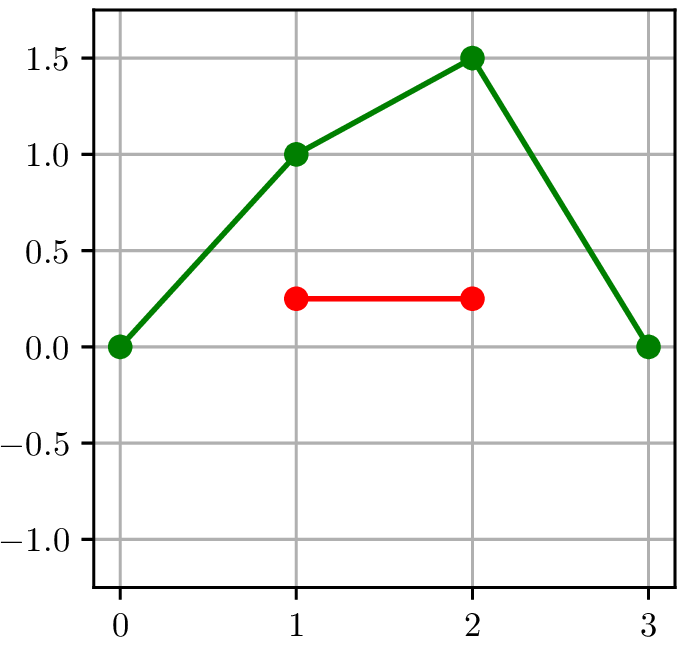}
		\caption{}
		\label{fig:div_2b}
	\end{subfigure}
	\begin{subfigure}[t]{0.3\textwidth}
		\centering
		\includegraphics[width=\textwidth]{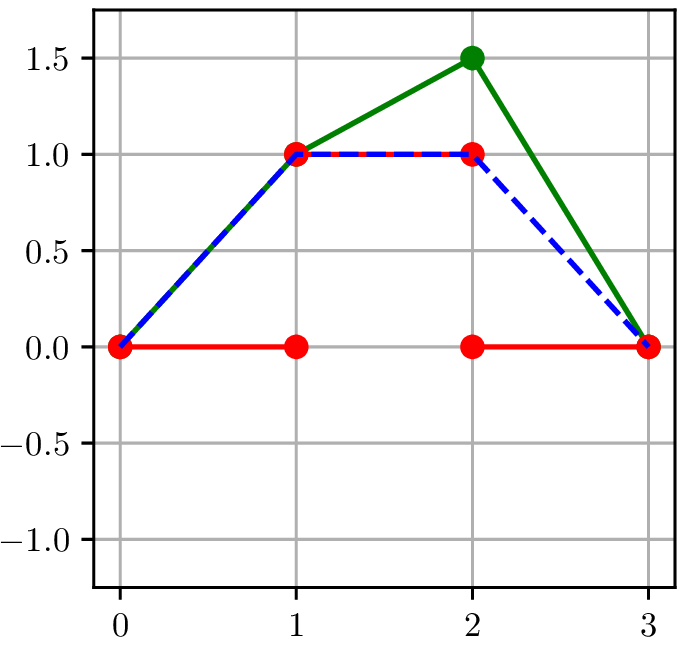}
		\caption{}
		\label{fig:div_2c}
	\end{subfigure}
	\caption{Division by $\max(x,0)$ (approximate result in blue).}
	\label{fig:div_2}
\end{figure}

As a final example, we can think of the case of a tropical polynomial in two dimensions. Let:
\begin{equation}
	p(x,y) = \max \left(2x, x+y+1, x+1, y+1, 1\right) , ~ d(x,y) = d(x) = \max(x,0)
\end{equation}
There are only three shifts available for $d(x)$, namely by adding $x, y, 0$ respectively, with corresponding vertical shifts $0, 1, 1$. Therefore, the quotient contains the following terms:
\begin{gather}
	q_{00} = 1+\max(x,0) = \max(x+1,1) \nonumber \\
	q_{01} = 1+y+\max(x,0) = \max(x+y+1,y+1) \nonumber \\
	q_{10} = x+\max(x,0) = \max(2x,x)
\end{gather}
In this case, the equality holds (as shown in Figure \ref{fig:div_3}), so there is no remainder.

\begin{figure}[h]
	\centering
	\begin{subfigure}[t]{0.5\textwidth}
		\centering
		\includegraphics[width=\textwidth]{./img/Figure_ENewt_3d.eps}
	\end{subfigure}%
	\begin{subfigure}[t]{0.5\textwidth}
		\centering
		\includegraphics[width=\textwidth]{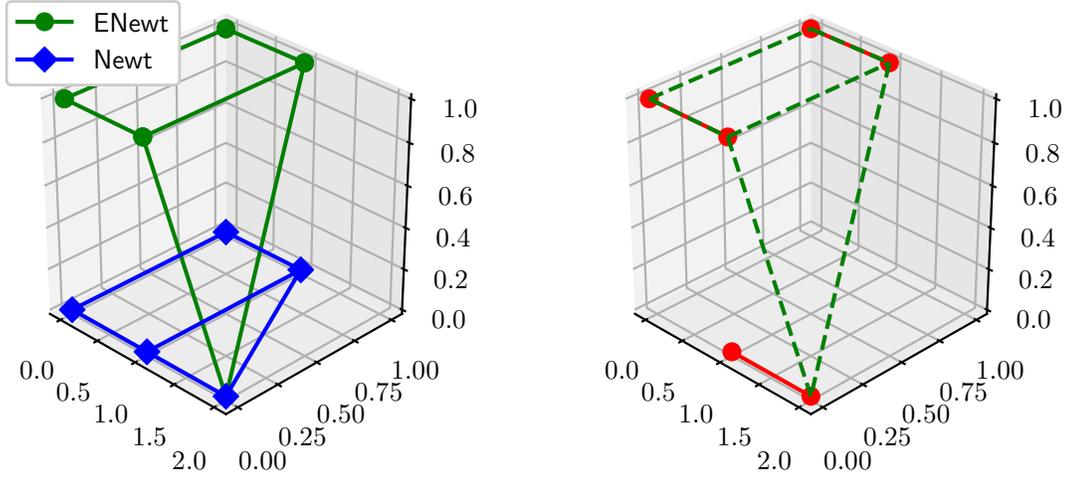}
	\end{subfigure}
	\caption{Tropical polynomial division in two dimensions.}
	\label{fig:div_3}
\end{figure}

\subsection{Formulation of the Algorithm via GGP}

In the case where the coefficients of the tropical polynomial $p(\bm{x})$ are positive (which can be done without loss of generality, since it is possible to add a large enough value to the dividend, and subtract it from the result, without the latter being affected), tropical polynomial division can also be formulated as a Generalized Geometric Programming (GGP) problem, as described in \cite{Boyd_GeomProg_07}. Indeed, the quotient of tropical polynomial division, as defined previously, can also be found by solving the following GGP:
\begin{gather*}
\mathrm{minimize} ~\sum_{c \in C} \left\lbrace l_c^{-1} \right\rbrace \\
\begin{matrix}
s.t.  & l_c q_c^{-1} \leq 1 , ~ \forall c \in C\\
~ & (q \oplus d)_j p_j^{-1}\leq 1 , ~ \forall j \in \mydeg(p(\bm{x}))
\end{matrix}
\end{gather*}
over the variables $q_c$ and $l_c$. We note that we define:
\begin{equation}
	(q \oplus d)_j = \max\limits_{c: j - c \in \mydeg\left(d(\bm{x})\right)} \left(q_c + d_{j-c} \right)
\end{equation}
as the \textit{max-plus convolution} of the coefficient sequences $d_i$ and $q_c$ of $d(\bm{x})$ and $q(\bm{x})$ respectively. This max-plus convolution also gives us the coefficients of $q(\bm{x})+d(\bm{x})$, similar to how the coefficients of regular polynomial products are computed via a regular convolution. It is obvious that the above problem is a GGP, since the minimization target and both types of inequality constraints are generalized posynomials of the variables $l_c$ and $q_c$.

\begin{theorem}
	The solution to the above GGP problem is equal to the polynomial $q(\bm{x})$ returned by the previous algorithm.
\end{theorem}

\begin{proof}
	The second type of constraints is equivalent to the condition that the shifted version of $q(\bm{x})$ has an Extended Newton Polytope below that of $p(\bm{x})$. The first type serves to bound the values $q_c$ from below, therefore forcing them to increase, since the goal decreases as that bound becomes higher.
	
	We can see that, due to the second type of inequality constraint imposing an upper limit on the values of $q_c$, this also implies that there is an upper limit in the values of $l_c$, due to the first type of constraints. Given that the minimization goal is strictly decreasing for all $l_c$, the minimum is attained precisely when all variables attain their upper bound, as closely as possible. Therefore, the polynomial $q(\bm{x})$ constructed by the coefficients $q_c$ is maximal, thus exactly the one given from the algorithm that has been already presented.
\end{proof}

\section{Application in Neural Networks}\label{sec:appl}


In what follows, we shall consider that the networks studied contain one hidden layer with ReLU activations, and an output layer with one neuron, with sigmoid activation. Such a network is suited to solving binary classification problems, and, as demonstrated in \cite{ZNL18}, the function to which it corresponds is constructed by the difference of two tropical polynomials, each of which has a zonotope as its Newton Polytope (also seen in \cite{ChMa18}).

Indeed, the output of each of the $n_1$ neurons of the hidden layer, which has a weight matrix $W^1 = [w_{ij}^1]$ and corresponding bias vector $b^1$, is equal to:
\begin{equation}
	y_i = \max \left(\sum_{j=1}^{n_1} w_{ij}^1 x_j + b^1 , 0\right)
\end{equation}
therefore a tropical polynomial (with possibly negative degrees). With this in mind, the value calculated by the output neuron (before the sigmoid activation is applied), with weight vector $w^2$ and bias $b^2$, is the following:
\begin{align}\label{eq:nn_zonotope}
	\sum_{i = 1}^{n_1} w^2_{i} y_i &= \sum_{i = 1}^{n_1} w^2_{i} \max \left(\sum_{j=1}^{d}w^1_{ij} x_j + b^1_{i} , 0 \right) + b^2 \nonumber \\
	&= \sum_{i = 1}^{n_1} w^2_{+,i} \max \left(\sum_{j=1}^{d}w^1_{ij} x_j + b^1_{i} , 0 \right) - \sum_{i = 1}^{n_1} w^2_{-,i} \max \left(\sum_{j=1}^{d}w^1_{ij} x_j + b^1_{i} , 0 \right) + b^2
\end{align}
where, $w^2 = w^2_+ - w^2_-$ the splitting of the weights of the output layer into two vectors with non-negative entries. Therefore, the output function consists of two zonotopes (sum of rescaled and shifted line segments, corresponding to the neurons in the hidden layer). It should be noted that the bias of the output neuron can be ignored during the application of the division algorithm, given that the approximation is the same, if it is instead added to the algorithm's result.

As noted by \cite{ZNL18}, the above formulation of neural networks as tropical rational functions can be extended even to rational degrees (since it is possible to convert them to integers, by multiplying with a large enough factor, and rescaling the inputs to compensate). Thus, in the following, we shall ignore whether the weights of the network are integers, since we can treat every network represented with floating point arithmetic as if it had integer degrees (as is also the case in \cite{ChMa18}).

\subsection{Divisibility of Neural Network Tropical Polynomials}

The above formulation of the networks studied leads us directly to the following:

\begin{theorem}
	Each of the two tropical polynomials which construct the function of our neural network can be divided exactly by the tropical polynomial corresponding to any of the line segments from which it is constructed.
\end{theorem}

\begin{proof}
	This follows immediately from the formulation of the neural network, as the difference of two sums in (\ref{eq:nn_zonotope}), and the fact that our algorithm always returns the greater possible tropical polynomial that has the form of the expected result.
\end{proof}

\noindent This observation further indicates the relationship of the operation we introduced, with the building blocks of a neural network.

\subsection{Approximation with Loss Minimization}

Let us now assume that we want to approximate the Newton Polytope of a neural network with the architecture described above, which has been trained to minimize a loss function (such as the cross entropy loss, commonly used in binary classification problems). The value of this function can be thought of as $L(u)$, where $u$ is a linear term corresponding to the output of the network (without the sigmoid activation). If the weights are perturbed, the first layer will lead to a different linear term $u'$ corresponding to this particular input. Here we can make the assumption that the network was properly trained, in the sense that the weights correspond to a local minimum for the loss function. Hence, a small disturbance in these weights is expected to strictly increase this loss function, when computed over the whole dataset.

Stemming directly from the above, it is evident that the increase in the loss function caused by this approximation is minimized, when $u' \approx u$. In other works (such as \cite{Luo_2017_ICCV}, \cite{He+17}), this is among a list of different criteria by which irrelevant neurons are pruned. Given that these terms correspond to vertices in the polytope related to the network, it is also possible to regard this pruning criterion, in the context of tropical polynomial division.

In particular, we seek to adapt the algorithm already introduced, so that the division by the tropical polynomial $d(\bm{x})$ minimizes the sum of distances between the vertices activated by each input and their approximation (which results in the above increase in the loss function being minimized). This can be formulated as the below optimization problem, where $p(\bm{x})$ is the tropical polynomial corresponding to either the positive or the negative part of the network, and $D$ a given dataset:

\begin{equation}\label{goal:orig}
	\mathrm{minimize} \sum_{\bm{x} \in D} |q(\bm{x}) + d(\bm{x}) - p(\bm{x})|^2
\end{equation}

This optimization problem is, in fact, a special case of the max-linear fitting problem, discussed in detail in \cite{MaBo09}, where the function to approximate (that is, $p(\bm{x}) - d(\bm{x})$) is itself convex and piecewise linear. Thus, given a divisor polynomial, the algorithm described in the relevant work can be used (where the data points are alternatively assigned to terms of $q(\bm{x})$, and then each term of $q(\bm{x})$ is fit via least-squares), to calculate the quotient which provides the exact result to the above.

It should be noted that this distance between feature maps is different in principle from the one presented in \cite{He+17}. Indeed, here we attempt to approximate the output of a network using a given divisor polynomial, instead of minimizing the goal without the structure provided by tropical polynomial division.

\subsection{Direct Approximation of the Newton Polytope}

In this section, we shall provide an alternative approximation, which instead of seeking to minimize the difference in the activations caused by each input, attempts to directly match the vertices of the smaller network with those of the original. To that end, we shall seek to solve the following optimization problem:

\begin{equation}\label{goal:0}
	\mathrm{minimize} ~\sum_{j \in \mydeg\left(p(\bm{x})\right)} \left| (q \oplus d)_j - p_j \right|^2
\end{equation}
where, once again, $(q \oplus d)_j$ is the max-plus convolution of the coefficient sequences $d_i, q_c$. Thus, we want to minimize the difference in the coefficients of $q(\bm{x})+d(\bm{x}), ~p(\bm{x})$. The only variables are, again, the values $q_c$.

This is not a direct optimization of the criterion of distance between activated vertices (given that the ones corresponding to $q(\bm{x}) + d(\bm{x})$ change during the course of the algorithm). Rather, it is ideal only in the special case where the activated vertices correspond to the same degrees, in both $p(\bm{x})$ and $q(\bm{x}) + d(\bm{x})$ (a situation which might be more likely the smaller the divisor, given that the approximation will tend to have more vertices in the upper hull). However, it is closer in nature to the division algorithm already presented.

We shall now attempt to formulate the above problem as a GGP. First of all, we note the following property for terms of this particular form:
\begin{equation}
	\frac{\partial}{\partial y_i} \sum_{j} \left(y_j - c_j\right)^2 = 2\left(y_i - c_i\right)
\end{equation}
and that:
\begin{equation}
	\frac{\partial}{\partial y_i} \frac{\sum_{j} \left( y_j^2 + c_j^2 \right)}{\sum_{j} 2 c_j y_j} = \frac{2 y_i - 2 c_i }{\left(\sum_{j} 2 c_j y_j\right)^2}
\end{equation}
so the functions we differentiate attain their maxima and minima in the same values of $y$ (since their partial derivatives vanish at the same positions). Therefore, goal (\ref{goal:0}) will attain its minimum for the same $q(\bm{x})+d(\bm{x})$, as the following:
\begin{equation}\label{goal:1}
	\mathrm{minimize} ~\frac{\sum_{j \in \mydeg\left(p(\bm{x})\right)}  \left((q \oplus d)_j\right)^2 + p_j^2}{\sum_{ij \in \mydeg\left(p(\bm{x})\right)}  2 p_j \left((q \oplus d)_j\right)} 
\end{equation}
This is still not valid for a GGP, due to the fact that the denominator of the fractional goal is not a monomial. However, we shall attempt to approximate the solution, using the following method.

First of all, we introduce variables $l_c > 0$, such that:
\begin{equation}
	l_c \leq q_c \Rightarrow l_c q_c^{-1} \leq 0 , \forall c \in C
\end{equation}
The above are valid GGP constraints, and the following holds (assuming all $q_c$ and $d_i$ are positive, thus $(q \oplus d)_j \geq 0$):
\begin{gather}\label{eqn:lim_property}
	l_c \leq q_c < q_c + d_{j-c} \leq (q \oplus d)_j \leq \sum_{j \in \mydeg\left(p(\bm{x})\right)} (q \oplus d)_j \Rightarrow \\
	\frac{1}{\prod\limits_{c \in C : j-c \in \mydeg\left(d(\bm{x})\right)} l_c} > \frac{1}{\left(\sum_{j \in \mydeg\left(p(\bm{x})\right)} (q \oplus d)_j\right)^r} \nonumber
\end{gather}
where $r = |C|$. Moreover, we shall reintroduce the upper limits for the coefficients of $q(\bm{x}) + d(\bm{x})$, but this time we relax these constraints with factors $\xi_j$, that is:
\begin{equation}
	(q \oplus d)_j p_j^{-1} \xi_j^{-1} \leq 1 , ~ \forall j \in \mydeg\left(p(\bm{x})\right)
\end{equation}

With the above, we now formulate the GGP below, which attempts to approximate the solution to (\ref{goal:1}):
\begin{gather}\label{goal:2}
	\mathrm{minimize} ~\frac{\left(\sum_{j \in \mydeg\left(p(\bm{x})\right)}  \left((q \oplus d)_j\right)^2 + p_j^2\right)^r}{\prod_{c \in C : j-c \in \mydeg\left(d(\bm{x})\right)} l_c} + R \sum_{j \in \mydeg\left(p(\bm{x})\right)} \xi_j \\
	\begin{matrix}
	s.t & l_c q_c^{-1} \leq 0 , \forall c \in C \\
	~ & (q \oplus d)_j p_j^{-1} \xi_j^{-1} \leq 1 , ~ \forall j \in \mydeg\left(p(\bm{x})\right)
	\end{matrix} \nonumber
\end{gather}
where $R$ is a regularization factor. This is a valid GGP, due to both the goal and the constraints being posynomials (the values $p_j$ are always positive, due to the construction of the network).

It is evident that, due to (\ref{eqn:lim_property}), the value of the optimization goal in (\ref{goal:1}) is, when raised to $r > 0$ (which does not affect minimization), strictly lower than the first part of the goal of (\ref{goal:2}). The role of the variables $l_c$ is again to push $q_c$ as high as possible. However, it should be noted that the minima in (\ref{goal:1}) and (\ref{goal:2}) may differ, given that the solution to (\ref{goal:2}) may lead to the two goals being equal at the solution point, but the value produced may be suboptimal for (\ref{goal:1}). This is where the regularization comes into play, limiting the amount that the coefficients of the result can increase beyond those of $p(\bm{x})$, with the rigid constraints being fully reinstated if the value of $R$ is too large. While different values of $R$ may produce different solutions, a good one may be found simply by solving multiple instances of the problem.

The above problem is essentially an adaptation of our tropical polynomial division algorithm: the shifted versions of the divisor polynomial $d(\bm{x})$ are raised as much as possible (this time, up to $\xi_j p_j$), with the result being an approximation of the original tropical polynomial.

\subsection{Practical Application}

It is evident that the problems demonstrated above might be difficult to solve in the general case, since not only do the output variables of the algorithms we have presented scale exponentially with increasing dimension $d$, but some also require prior knowledge of the vertices and the faces of the Newton Polytope (finding only the former of which has the same scaling with $d$, \cite{Chazelle:1993:OCH:2805881.2830424}). However, in the case of approximating a neural network of the above form, we propose the following tractable method, of finding a suboptimal solution to minimizing either (\ref{goal:orig}) or (\ref{goal:0}).

\textbf{Neural Network Approximation} Let $p_{+}$, $p_{-}$ be the positive and negative parts of the network we want to approximate, with a smaller one consisting of $f\%$ neurons in the hidden layer. The algorithm used for this approximation is as follows:
	\begin{itemize}
		\item Phase 1:
		\begin{itemize}
			\item Randomly sample a subset $X$ of the training set.
			\item For each $x^i \in X$ calculate the corresponding vertex $u^i_{+}$, $u^i_{-}$ in the Extended Newton Polytope of each of $p_{+}$, $p_{-}$, activated by this input, that is $(u^i_{+})^T x_i = p_{+}(x_i)$, similarly for $u^i_{-}$.
			
			\item Sort each set of vertices in descending order, based on the number of their appearances.
			
			\item For the set of $u^i_{+}$, set the first neuron weight $w_1$ equal to the first vertex in the sorted list.
			
			\item For the j\textsuperscript{th} vertex $u^j_{+}$ in the list, up to $f\%$ of the neurons in the positive part, randomly pick one of the already assigned weights (let this be $w_k$), and set:
			\begin{gather*}
				w_k \leftarrow w_k - u^j_{+} \\
				w_j \leftarrow u^j_{+}
			\end{gather*}
			
			\item Repeat for the negative part of the network.
			
			\item Gather the weights created ($f\%$ of the original, in total) and use them as the first layer of the network, while assigning weights $\pm1$ in the output layer, for the positive and negative parts, respectively.
		\end{itemize}
		
		\item Phase 2:
		\begin{itemize}
			\item For each element of the set $X$, calculate the value corresponding to the activation of the output neuron in the original network, minus the same activation in the approximation network.
			
			\item Set a bias in the output neuron, equal to the mean of these values, plus the original bias of the output neuron.
		\end{itemize}
	\end{itemize}

In the above algorithm, the method presented in phase 1 allows us to create a network which matches perfectly the most important vertices of the polytope, therefore giving us a good polynomial to use as divisor. Moreover, this process also seeks to identify correlated neurons, which tend to fire together, intuitively leading to some level of regularization of the network.

As for the idea behind phase 2, we can show the following:
\begin{theorem}
	If $p(\bm{x})$ matches the degrees of $d(\bm{x})$, then $q(\bm{x}) = q_{\bm{0}}$, and the goal (\ref{goal:orig}), can be minimized if we set:
	\begin{equation}
		q_{\bm{0}} = \frac{1}{|X|} \sum_{x \in X} \left(p(\bm{x}) - d(\bm{x})\right)
	\end{equation}
\end{theorem}
\begin{proof}
	Assuming the situation described, the goal (\ref{goal:orig}) takes the following form:
	\begin{equation}
		\sum_{\bm{x} \in X} \left| q_{\bm{0}} + d(\bm{x}) - p(\bm{x}) \right|^2 = \sum_{\bm{x} \in X}  \left| q_{\bm{0}} - \left(p(\bm{x}) - d(\bm{x})\right) \right|^2
	\end{equation}
 	This is equivalent to the sum of squared distances from a set of points in the real line, for which the minimum is known to be attained precisely at the mean of these points.
\end{proof}
\noindent Due to the above, in phase 2 of the algorithm, we also add the appropriate bias, in order to compensate for the vertices ignored during phase 1.

Regarding the time complexity of the algorithm, it is obvious that the number of steps for phase 1 scales as $|X| \mathrm{log}|X|+N^2$ (where $N$ the number of remaining neurons), while the number of steps required for phase 2 is again linear in $|X|$. Therefore, assuming a reasonable amount of time necessary for the calculation of the network outputs for these samples, the method presented is tractable. It is also likely faster than training a new network, a process which requires at least $|X|$ updates for each of the $N$ neurons.

\section{Experiments}\label{sec:exper}

We shall now apply the method discussed previously to neural networks already trained in a binary classification problem. In particular, we shall use the following datasets:
\begin{itemize}
	\item The IMDB Movie Review dataset \cite{maas-EtAl:2011:ACL-HLT2011}, which contains 50000 movie reviews, half of them being in the training set and the other half in the test set. We shall utilize the preprocessing proposed in the above (keeping 5000 words, except the 50 most common ones), while also truncating or padding each sequence, so that all have a fixed length of 200.
	
	\item The MNIST dataset \cite{LeBo18}, where at first only the pairs of digits 4 - 9  and 3 - 6 are considered, for a binary classification problem (even in a highly accurate model like \cite{Oy+18}, these classes can still be confused). Afterwards, the entirety of the dataset is split into even and odd numbers, for a more general binary classification problem.

\end{itemize}

For our experiments, we made use of the Python library Keras \cite{chollet2015keras}. The models and the basic architectures provided were used in order to demonstrate the results of the application of our minimization method. The reported accuracies are taken as the average between the different runs of each experiment.

\subsection{IMDB Movie Review - Original Method}

In this experiment, we examined 3 different models, trained in the IMDB dataset, using an 80\%-20\% train-validation split, with the validation loss being used to stop the training of the models, so as to avoid overfitting. All three models learned an embedding of the words of the reviews in a 50-dimensional space. The difference lies in the representation used afterwards:
\begin{itemize}
	\item The first model simply fed all available points in a fully connected network, with one hidden layer containing 1024 neurons.
	
	\item The second and the third model employed a bidirectional LSTM with 32 units and a one-dimensional CNN (with two convolutional layers of 8 and 4 units, with ReLU activations, kernels of 3 pixels and max-pooling of 5 pixels), respectively, to create a smaller representation to use as input to a fully connected network, with 256 neurons in the hidden layer.
\end{itemize}

It should be noted that we opt not to use pretrained word embeddings (such as those provided by \cite{pennington2014glove}), given that our models are not complex enough to take advantage of the benefits of using these embeddings. Moreover, the above architectures are not optimal, but they demonstrate the value of our method.

Our algorithm was applied to the above models, using a subset of 2000 points from the training set. Various percentages of the neurons of the hidden fully connected layer were kept, and the relevant results are shown in Table \ref{table:1}, where the accuracy of each model in the test set is reported.
\begin{table}[h]
	\centering
	\begin{tabularx}{\textwidth} {
			| >{\centering\arraybackslash}X
			| >{\centering\arraybackslash}X
			| >{\centering\arraybackslash}X
			| >{\centering\arraybackslash}X
			| >{\centering\arraybackslash}X
			| >{\centering\arraybackslash}X
			| >{\centering\arraybackslash}X |		
		}
		\hline
		Percentage of Neurons Kept & \multicolumn{2}{|c|}{FC (1024)} & \multicolumn{2}{|c|}{LSTM + FC (256)} & \multicolumn{2}{|c|}{1D CNN + FC (256)} \\
		\hline
		100\% (Original) & \multicolumn{2}{|c|}{84.056\%} & \multicolumn{2}{|c|}{85.006\%} & \multicolumn{2}{|c|}{83.927\%} \\
		\hline
		~ & Reduced & Fresh Training & Reduced & Fresh Training & Reduced & Fresh Training \\
		\hline
		90\% & 84.299\% & 83.647\% & 85.190\% & 84.282\% & 84.095\% & 84.296\% \\
		\hline
		75\% & 84.263\% & 83.668\% & 85.108\% & 84.427\% & 84.059\% & 84.212\% \\
		\hline
		50\% & 84.314\% & 83.808\% & 85.029\%& 84.430\% & 84.006\% & 83.615\% \\
		\hline
		25\% & 84.366\% & 84.108\% & 85.092\% & 84.039\% & 83.930\% & 83.662\% \\
		\hline
		10\% & 84.333\% & 84.700\% & 85.126\% & 84.914\% & 84.042\% & 84.300\% \\
		\hline		
		5\% & 84.362\% & 84.744\% & 85.079\% & 83.919\% & 84.118\% & 84.675\% \\
		\hline		
		2\% & 84.314\% & 84.500\% & 84.878\% & 84.081\% & 83.922\% & 83.643\% \\
		\hline
		1\% & 84.206\% & 85.180\% & 85.017\% & 82.566\% & 83.915\% & 81.855\% \\
		\hline
		0.5\% & 83.860\% & 85.564\% & - & - & - & - \\
		\hline
	\end{tabularx}
	\caption{Average accuracy on test set (IMDB).}
	\label{table:1}
\end{table}
These results demonstrate the value of our method, given that using the full algorithm leads to a drop in performance of less than 0.5\% at most, an arguably negligible cost when considering the fact that it allows us to reduce the size of the fully connected part of the network by a factor of up to 100. In some cases, it even improves upon the full network.

We also report the results of applying minimization by approximating not the activations, but the biases of the activated neurons, therefore attempting to minimize goal (\ref{goal:0}) over the dataset. The results are in Table \ref{table:2}.
\begin{table}[h]
	\centering
	\begin{tabularx}{0.8\textwidth} {
			| >{\centering\arraybackslash}X
			| >{\centering\arraybackslash}X
			| >{\centering\arraybackslash}X
			| >{\centering\arraybackslash}X |		
		}
		\hline
		Percentage of Neurons Kept & FC (1024) Accuracy (\%) & LSTM + FC (256) Accuracy (\%) & 1D CNN + FC (256) Accuracy (\%) \\
		\hline
		100\% (Original) & 84.056 & 85.006 & 83.927 \\
		\hline
		90\% & 83.804 & 84.988& 83.326\\
		\hline
		75\% & 84.413& 84.483& 82.812\\
		\hline
		50\% & 84.208& 84.097& 82.763\\
		\hline
		25\% & 84.140& 84.774& 83.186\\
		\hline
		10\% & 84.023& 83.159 & 83.469\\
		\hline
		5\% & 82.723& 84.883& 83.078\\
		\hline
		2\% & 83.684& 85.014& 83.928\\
		\hline
		1\% & 79.551& 84.997& 83.922\\
		\hline
		0.5\% & 83.971& -&- \\
		\hline	\end{tabularx}
	\caption{Average accuracy on test set, goal (\ref{goal:0}) (IMDB).}
	\label{table:2}
\end{table}
This method has worse results than the one used previously, likely because the approximation uses one large divisor (so the vertices of the polytope of the approximation may not match those of the original). Nevertheless, in the case of the first and the second model, they can be deemed satisfactory, since they lead to a drop of 4.5\% in only a single case.

In Table \ref{table:1}, the accuracies of models with less neurons, trained from scratch, are also reported. The results for the freshly trained models are better in the case of the fully connected model, possibly due to the number of parameters making them able to approximate the correct result, no matter the number of neurons. However, in the case of the other two models, the accuracy achieved by the freshly trained model is lower in general, and especially when only a few neurons of the hidden layer are kept. Given that the input to the fully connected layer is very small in these models, this might mean that the two classes are entangled in a convoluted fashion in smaller dimensions. Therefore, our method is useful in the case where the number of features is small, since it allows us to get information from a higher dimensional projection, before minimizing the network.

In theory, a freshly trained model should be able to achieve at least as good results as our own, given a good optimization algorithm. However, the results from Table \ref{table:time} show that our method is significantly faster than training all different models. Therefore, our method is preferable, over training several different models, while also obtaining adequate results.

\begin{table}[h]
	\centering
	\begin{tabularx}{0.8\textwidth} {
			| >{\centering\arraybackslash}X
			| >{\centering\arraybackslash}X
			| >{\centering\arraybackslash}X |		
		}
		\hline
		Model & Our method runtime - full algorithm (sec) & Time required to train all models (sec)\\
		\hline
		FC & 86.9 & 151.9\\
		\hline
		LSTM+FC & 139.7& 907.0\\
		\hline
		CNN+FC & 18.3& 149.2\\
		\hline	\end{tabularx}
	\caption{Runtime comparison (IMDB).}
	\label{table:time}
\end{table}

It should be noted that, in the case of the CNN, using linear instead of ReLU activations in the convolutional networks led to a significant accuracy drop, both with the full network and after minimization. This is possibly due to the fact that using these activations prevent the network from attaining a minimum for the loss, therefore close approximation of the network values may not guarantee the best results possible.

\subsection{IMDB Movie Review - Iterative Method}

In an attempt to enhance the results demonstrated in Table \ref{table:1}, we create an iterative version of our algorithm, where the minimization is split into distinct steps. In each step, the amount of neurons in the hidden layer is halved, using the method to minimize goal (\ref{goal:orig}) as before, but the two fully connected layers are retrained for 10 epochs after each iteration, in order to compensate for the loss of accuracy caused by the reduction.


\begin{table}[h]
	\centering
	\begin{tabularx}{\linewidth} {
			| >{\centering\arraybackslash}X
			| >{\centering\arraybackslash}X
			| >{\centering\arraybackslash}X |		
		}
		\hline
		Neurons Kept & FC (1024) & CNN+FC (256)  \\
		\hline
		100\% & 84.056\% & 83.927\% \\
		\hline
		50\% - Iter. 1 & 84.334\% & 84.080\% \\
		\hline
		25\% - Iter. 2 & 84.379\% & 84.094\% \\
		\hline
		12.5\% - Iter. 3 & 84.347\% & 84.105\% \\
		\hline
		6.25\% - Iter. 4 & 84.272\% & 84.131\% \\
		\hline
		3.1\% - Iter. 5 & 84.235\% & 84.049\% \\
		\hline
		1.6\% - Iter. 6 & 84.299\% & 84.054\% \\
		\hline
		0.8\% - Iter. 7 & 84.192\% & - \\
		\hline
	\end{tabularx}
	\caption{Average accuracy on test set, iterative method with retraining (IMDB).}
	\label{table:4}
\end{table}

The results can be seen in Table \ref{table:4}, for the first and the third model (the second appears to have no room for improvement). This is, in general, a slight improvement over the results of Table \ref{table:1}, since retraining the network allows it to compensate for any mistakes made in its approximation. It is possible that better results might have been obtained with a more careful retraining step, in order to better approximate a local minimum. Moreover, the results for the convolutional model are better than those of the freshly trained networks, since we also take advantage of a good representation learned when using the large model. This is not the case for the first model, likely for similar reasons as above.

\subsection{MNIST Dataset - Pairs 4/9, 3/6}

In this experiment, we filter both the training and test sets of the MNIST dataset, in order to keep only the samples corresponding to the pairs considered. The model examined in this case consists of two convolutional layers, of the same number of units as before, as well as two fully connected layers, with the hidden layer containing 1000 neurons. This way, we can again learn a small representation for our data. The method used is, again, the one which minimizes goal (\ref{goal:orig}). The results can be found in Table \ref{table:5}.

\begin{table}[h]
	\centering
	\begin{tabularx}{\linewidth} {
			| >{\centering\arraybackslash}X
			| >{\centering\arraybackslash}X
			| >{\centering\arraybackslash}X
			| >{\centering\arraybackslash}X
			| >{\centering\arraybackslash}X |		
		}
		\hline
		Neurons Kept &Dig. 3-5, Reduced &Dig. 3-5, Fresh & Dig. 4-9, Reduced & Dig. 4-9, Fresh \\
		\hline
		100\% &99.180 & - &99.046 & -\\
		\hline
		75\% &99.138 &99.380 & 99.036& 99.186\\
		\hline
		50\% &99.106 &99.127 & 99.046& 99.056\\
		\hline
		25\% &99.117 &99.138 & 98.985& 98.905\\
		\hline
		10\% &99.106 &99.190 & 99.005& 99.166\\
		\hline
		2\% &99.117 &99.096 & 98.905& 99.066\\
		\hline
		1\% &99.180 &99.138 & 98.805& 98.945\\
		\hline
		0.5\% &99.180 &99.117 & 99.005& 99.026\\
		\hline
	\end{tabularx}
	\caption{Average accuracy on test set (MNIST).}
	\label{table:5}
\end{table}

It can be seen that the performance drop caused by our method is again small, representing a drop of less than 0.5\%, with similar reduction of network size. Moreover, the freshly trained models do exhibit similar behavior as before, albeit on a smaller scale. This is possibly caused by the ease of this problem, since we do not benefit much from learning a better representation by training on a larger network (it could be the case that the classes are easily distinguishable). However, our method still has the aforementioned benefits of allowing us to train larger models, which are more likely to converge, and obtain adequate results, without training a large number of networks.

\subsection{MNIST Dataset - Even/Odd}

Here, we use the entirety of the MNIST dataset for our experiment. The two classes considered are even and odd digits, with each sample being relabeled appropriately. The network used is the same as the previous experiment, with the only difference being the use of slightly larger kernels (of size 5, as opposed to 3). The results can be seen in Table \ref{table:evenodd}.

 \begin{table}[h]
 	\centering
 	\begin{tabularx}{\linewidth} {
 			| >{\centering\arraybackslash}X
 			| >{\centering\arraybackslash}X
 			| >{\centering\arraybackslash}X	|
 		}
 		\hline
 		Neurons Kept &Reduced &Fresh \\
 		\hline
 		100\% &98.156\% & -\\
 		\hline
 		75\% &98.062\% &98.182\% \\
 		\hline
 		50\% &97.914\% &98.370\% \\
 		\hline
 		25\% &98.030\% &98.196\% \\
 		\hline
 		10\% &97.994\% &98.014\% \\
 		\hline
 		2\% &98.030\% &98.392\% \\
 		\hline
 		1\% &98.034\% &97.986\% \\
 		\hline
 	\end{tabularx}
 	\caption{Average accuracy on test set (MNIST, even/odd).}
 	\label{table:evenodd}
 \end{table}

Once again, the approximation performed by our method leads to minimal loss of accuracy. Furthermore, the results are very close to those attained with the freshly trained models. Hence, the behavior is similar to that of the previous experiments. Of note, however, is the fact that the kernel of the convolutional layers was enlarged in this case. This was necessary due to the high variance of samples in each class (due to each corresponding to visually different digits). Thus, a better representation was needed, via a larger convolutional part, for the approximation to be acceptable.

\section{Conclusions and Future Work}\label{sec:conclusions}

In this work, we examined a framework for a form of tropical polynomial division, and made use of the intuition it provides in order to define a method to approximate two-layer fully connected networks trained for a binary classification problem. The results obtained demonstrate that the Newton Polytopes of the tropical polynomials corresponding to the network may provide a reliable way to approximate its results, due to them encoding the totality of the information contained in the network.

In the future, we hope to generalize this method, in order to apply it to multiclass problems, as well as more general architectures (such as convolutional neural networks), and deep models, possibly via its recursive application between layers. Moreover, further study of this method must be made, along with experimentation in a greater variety of datasets, in order to identify the extent of its applications, as well as further examination of this type of division of polynomials in the max-plus semiring, in order to identify more of its properties.

\newcommand{\noopsort}[1]{} \newcommand{\printfirst}[2]{#1}
  \newcommand{\singleletter}[1]{#1} \newcommand{\switchargs}[2]{#2#1}

\bibliographystyle{ieeetr}

\end{document}